\documentclass[11pt]{article}
\usepackage[preprint]{plain}
\usepackage{times}
\usepackage{latexsym}
\usepackage[T1]{fontenc}
\usepackage[utf8]{inputenc}
\usepackage{microtype}
\usepackage{inconsolata}
\usepackage{graphicx}
\usepackage{amsmath}
\usepackage{booktabs}
\usepackage{amsthm}
\newtheorem{proposition}{Proposition}

\title{Augmenting Bias Detection in LLMs Using Topological Data Analysis}

\author{Keshav Varadarajan \\ Department of Computer Science \\ Duke University \And
        Tananun Songdechakraiwut \\ Department of Computer Science \\ Duke University}

\begin{document}

\newcommand{\set}[1]{\{ #1 \}}
\newcommand{\series}[2]{\sum_{k = #1}^{#2}}

\maketitle
\begin{abstract}
Recently, many bias detection methods have been proposed to determine the level of bias a large language model captures. However, tests to identify which parts of a large language model are responsible for bias towards specific groups remain underdeveloped. In this study, we present a method using topological data analysis to identify which heads in GPT-2 contribute to the misrepresentation of identity groups present in the StereoSet dataset. We find that biases for particular categories, such as gender or profession, are concentrated in attention heads that act as hot spots. The metric we propose can also be used to determine which heads capture bias for a specific group within a bias category, and future work could extend this method to help de-bias large language models.
\end{abstract}

\section{Introduction}

As large language models have developed, they have become increasingly important for tasks such as machine translation, question answering, information retrieval, text summarization, word sense disambiguation, entity linking, semantic role labeling, and natural language inference. This has become especially true with the advent of pre-trained models like BERT \citep{devlin2019bert}, GPT \citep{radford2019language}, and BART \citep{lewis2019bart}, which are able to take advantage of large datasets like Wikipedia \citep{hovy2013collaboratively}.

However, the uncurated nature of the datasets used to train these models can result in biased representations of certain groups \citep{bender2021StochasticParrots}. Examples include the presence of gender bias in model outputs \citep{NEURIPS2021_1531beb7, kotek2023GenderBias} and the presence of stereotypes about Muslims \citep{abid2021AntiMuslim}.

Many studies have attempted to measure the bias in these models using the probability the model assigns to different sentences \citep{webster2020GenderCorrelations, ahn-oh-2021-mitigating, Kaneko_Bollegala_2022, nadeem-etal-2021-stereoset}, the word or sentence embeddings the model creates \citep{aylin2017Semantics, guo2021IntersectionalBiases, may-etal-2019-measuring, dev-etal-2021-oscar}, or the generated text of the model \citep{Chowdhery2022PaLMSL, chung2022scaling, gehman-etal-2020-realtoxicityprompts, Liang2022HolisticEO}. One issue with these methods is that they do not allow us to determine which parts of the model contribute to the bias it learns.

Recently, \citet{unlearningBiasGradients} proposed a contrastive gradient-based solution to determine which parameters of the model contribute to bias. However, this method applies to single tokens found in a sentence and does not capture the biases between sentences in the model.

One promising method to solve this problem is topological data analysis \citep{wasserman2018TDA, ghrist2008barcodes}. In the past, this approach has been used to study networks such as images and brains \citep{songdechakraiwut2020dynamic, hu2019topology}. Recently, it has been used to learn topological features in natural language data for tasks such as text classification \citep{Gholizadeh2020, Wen_2020}.

Our contributions are as follows:
\begin{itemize}
    \item We propose a new metric for the bias captured by each self-attention head of a large language model using their learned topological features.
    \item We find self-attention head hot spots in GPT-2 that capture more bias for given categories.
    \item We find that GPT-2 captures more bias in the gender and race categories.
\end{itemize}

The remainder of this paper is organized as follows: Section \ref{section:preliminaries} covers essential background, including epsilon filtrations and the Wasserstein distance. Section \ref{section:methodology} details our methods and introduces the Wasserstein bias statistic. Results are presented in Section \ref{section:results}, with directions for future work in Section \ref{section:conclusions}. Finally, Sections \ref{section:limitations} and \ref{section:ethics} discuss study limitations and ethical considerations.

\section{Preliminaries}
\label{section:preliminaries}

\subsection{Persistent homology of graphs}

Let $G = (V, E)$ be a complete graph consisting of a set of nodes $V$ and a symmetric weighted adjacency matrix $E$ with unique entries, assuming that each pair of nodes without an edge between them has an edge with infinitesimal non-zero weight. Let $G_\epsilon = (V, E_\epsilon)$ be a binary graph resulting from thresholding a graph $G$ at a threshold value $\epsilon$ such that:
\[ 
    E_\epsilon = E_{\epsilon, i, j} = \begin{cases} 
      1 & E_{i, j} > \epsilon \\
      0 & E_{i, j} \leq \epsilon.
   \end{cases}
\]
That is, a binary $G_\epsilon$ is an unweighted graph with the node set $V$ but edges of weights less than or equal to a threshold value $\epsilon$ removed. The binary graph is viewed as a simplicial complex consisting of only nodes and edges, known as a 1-skeleton \citep{edelsbrunner2022computational}.
Consider a graph filtration \citep{lee2012persistent} defined by a series of binary graphs 
\begin{equation}
    G_{\epsilon_0} \supset G_{\epsilon_1} \supset G_{\epsilon_2} \supset \ldots \supset G_{\epsilon_k},
    \label{eq:graphfiltration}
\end{equation}
where $\epsilon_0 < \epsilon_1 < \ldots < \epsilon_k$ are threshold values. Notice that as $\epsilon$ increases, more edges are thresholded from a graph $G$.

Persistent homology tracks the birth and death of topological features as threshold values $\epsilon$ vary. A topological feature that appears at a threshold value $b_i$ and persists until $d_i$ is represented as a point $(b_i, d_i)$ in a plane, forming a persistence diagram \citep{edelsbrunner2008persistent}. In the 1-skeleton, the only non-trivial topological features are connected components (0-dimensional) and cycles (1-dimensional). There are no higher-dimensional topological features in the 1-skeleton, unlike more general simplicial complexes \citep{ghrist2008barcodes}. This approach ensures computational scalability, which is crucial for performing statistical significance testing in topological analysis.

Consider the graph filtration defined in Eq. (\ref{eq:graphfiltration}), which begins with a complete graph $G_{-\infty}$ and proceeds by progressively removing edges, one at a time as the threshold parameter $\epsilon$ increases, ending with an edgeless graph $G_{\infty}$. As $\epsilon$ increases, the number of connected components and cycles change monotonically: connected components increase while cycles decrease \citep{songdechakraiwut2021topological, Songdechakraiwut2023-qy}. Once connected components appear, they persist until the final state $G_{+\infty}$, resulting in all connected components having death values at $+\infty$. This allows us to simplify the representation of connected components to a collection of birth values $I_0(G) = \{\epsilon_{b,i}\}$. Conversely, all cycles are present in the complete graph $G_{-\infty}$ and thus have birth values at $-\infty$. Therefore, cycles are represented by a collection of death values $I_1(G) = \{\epsilon_{d,i}\}$.

\subsection{Wasserstein distance metric}

The Wasserstein distance between the simplified persistence diagrams of the graph filtration defined in Eq. (\ref{eq:graphfiltration}) can be obtained using a closed-form solution. Given a graph $G$, its underlying probability density function on the persistence diagram is defined as Dirac masses \citep{turner2014frechet}:
\[
    f_{G, k}(x) = \frac{1}{|I_k(G)|} \sum_{\epsilon \in I_k(G)} \delta(x-\epsilon), 
\]
where $\delta(x-\epsilon)$ is a Dirac delta function centered at $\epsilon$, $|I_k(G)|$ is the cardinality of $I_k(G)$, and $k \in \{0, 1\}$ denotes the dimension of a topological feature, i.e., $k=0$ for connected components and $k=1$ for cycles. The corresponding empirical cumulative distribution function is defined as:
\[
    F_{G, k}(x) = \frac{1}{|I_k(G)|} \sum_{\epsilon \in I_k(G)} \mathbf{1}_{\epsilon \leq x}, 
\]
where $\mathbf{1}_{\epsilon \leq x}$ is the indicator function that is 1 if $\epsilon \leq x$ and 0 otherwise. That is, $F_{G, k}(z)$ is a step function with $|I_k(G)|$ steps.

The pseudoinverse of $F_{G, k}$, denoted as $F^{-1}_{G, k}(z)$, is defined to be the smallest $x$ for which $F_{G, k}(x) \geq z$.
The k-dimensional Wasserstein distance between two graphs $G_1$ and $G_2$ is defined to be
\[
    d_k(G_1, G_2) = \int_0^1 (F^{-1}_{G_1, k}(z) - F^{-1}_{G_2, k}(z))^2 dz.
\]
We define an approximation of the pseudoinverse distribution $F^{-1}_{G, k}(z)$ with a smoothing parameter $N \in \textbf{N}$ as a step-function $f^{-1}_{G, k, N}:[0,1] \rightarrow \textbf{R}$ with $N$ steps:
\[ 
    f^{-1}_{G, k, N}(z) = \begin{cases} 
      N\int_{0}^{\frac{1}{N}} F^{-1}_{G, k}(t) dt & 0 \leq z < \frac{1}{N} \\
      N\int_{\frac{1}{N}}^{\frac{2}{N}} F^{-1}_{G, k}(t) dt & \frac{1}{N} \leq z < \frac{2}{N} \\
      ... \\
      N\int_{\frac{N - 1}{N}}^{1} F^{-1}_{G, k}(t) dt & \frac{N - 1}{N} \leq z \leq 1.
   \end{cases}
\]
We can see an example of an approximated pseudoinverse distribution in Figure \ref{fig:pseudoinverseApprox}. We can think of each step of the approximated distribution as a weighted average of the corresponding interval in the actual distribution.
We use this approximation to calculate the approximated Wasserstein distance between two graphs as
\[
    d'_k(G_1, G_2) = \int_0^1 (f^{-1}_{G_1, k}(z) - f^{-1}_{G_2, k}(z))^2 dz.
\]

\begin{figure*}[t]
  \centering
  \includegraphics[width=.65\textwidth]{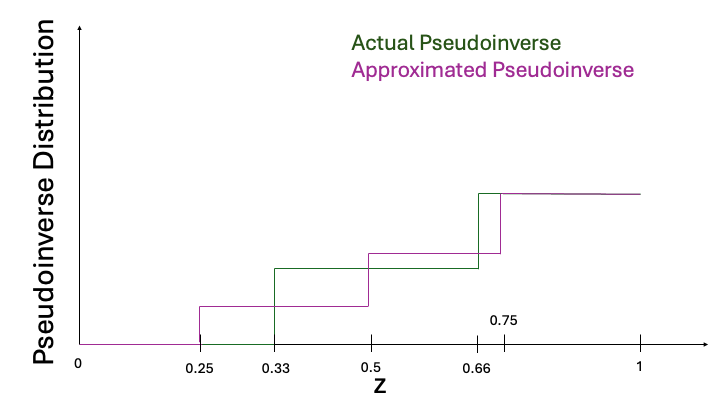}
  \caption{An example of an approximated pseudoinverse distribution with the smoothing parameter $N = 4$. Here the approximated pseudoinverse is purple and the actual is green.}
  \label{fig:pseudoinverseApprox}
\end{figure*}

\subsection{Cluster centers and variance}

Let $H$ be a set of complete graphs and let $N = \text{max}\{ |I_k(G) | G \in H \}$ be the smoothing parameter used for approximating each pseudo-inverse distribution.
We define the k-dimensional cluster center of $H$ to be the graph $\bar{G}$ constructed such that the pseudoinverse of $F_{\bar{G}, k}$ is the average of the pseudoinverse distributions of the cluster so
\begin{align*}
    F^{-1}_{\bar{G}, k, N}(z) & = \frac{1}{|H|} \sum_{G \in H} F^{-1}_{G, k, N}(z).
\end{align*}
We define the k-dimensional variance of the cluster to be the average Wasserstein distance from the cluster center to any graph in the cluster
\begin{align*}
    \Sigma_{H,k,N}^2 & = \frac{1}{|H|} \sum_{G \in H} d(G, \bar{G}).
\end{align*}
However, for the purposes of implementation we use the approximations of the pseudoinverse distributions so the approximated pseudoinverse of the cluster center becomes
\begin{align*}
    f^{-1}_{\bar{G}, k, N}(z) & = \frac{1}{|H|} \sum_{G \in H} f^{-1}_{G, k, N}(z)
\end{align*}
and the approximated cluster variance becomes 
\begin{align*}
    \sigma_{H,k,N}^2 & = \frac{1}{|H|} \sum_{G \in H} d'(G, \bar{G}).
\end{align*}

\section{Methodology}
\label{section:methodology}

We explain the methodology used to create our bias test in this section.

\subsection{Dataset and models}

In this study, we evaluate the bias captured by GPT-2 \citep{radford2019language}. We used the inter-sentence category of the StereoSet dataset \citep{nadeem-etal-2021-stereoset} for this bias test. It contains 2,123 discourse examples that begins with mentioning a group and continues with three options that either contain a positive stereotype (anti-stereotype), a negative stereotype, or an irrelevant sentence. Ideally, the model should equally learn the positive and negative stereotypes. StereoSet includes data for gender, profession, race, and religion related bias.

\subsection{Definitions}

Below, we provide some definitions used in the paper.

\begin{itemize}
    \item We define $D_{source}, D_S, D_A, D_I$ as the set of context sentences, stereotype continuations, anti-stereotype continuations, and irrelevant continuations respectively.
    \item We define $M_S^{l, h}, M_A^{l, h}, M_I^{l, h}$ as the set of attention matrices from the $l$-th layer and $h$-th head of the model generated by concatenating the context sentences with the stereotype continuations, anti-stereotype continuations, and irrelevant continuations respectively. An example of one such matrix can be found in Figure \ref{fig:attentionExample}.
    \item We define $C_S^{l, h}, C_A^{l, h}, C_I^{l, h}$ as the subsets of $M_S^{l, h}, M_A^{l, h}$ and $M_I^{l, h}$ generated using the same context sentences.
    \item We define $\Theta_S^{l, h, k}, \Theta_A^{l, h, k}, \Theta_I^{l, h, k}$ as the set of pseudo-inverse distributions of k-dimensional topological features for graphs calculated using the matrices from from $C_S^{l, h}, C_A^{l, h}$ and $C_I^{l, h}$ respectively.
    \item We define $\bar{f_S}^{l, h, k}, \bar{f_A}^{l, h, k}, \bar{f_I}^{l, h, k}$ as the k-dimensional pseudoinverse distributions of the approximated cluster centers of graphs generated from $C_S^{l, h}, C_A^{l, h}$ and $C_I^{l, h}$ respectively.
    \item We define $\sigma^2_{S, l, h, k}, \sigma^2_{A, l, h, k}, \sigma^2_{I, l, h, k}$ as the k-dimensional approximated cluster variances of graphs generated from $C_S^{l, h}, C_A^{l, h}$ and $C_I^{l, h}$ respectively.
\end{itemize}

\begin{figure*}[t]
  \centering
  \includegraphics[width=.7\textwidth]{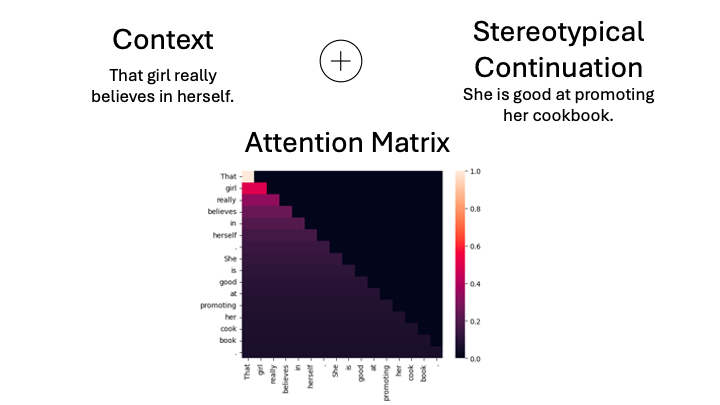}
  \caption{An example of an attention matrix generated for a stereotype sentence from the first layer and first head of GPT-2.}
  \label{fig:attentionExample}
\end{figure*}

\subsection{Wasserstein bias statistic}

We define the k-dimensional \emph{Wasserstein Bias Statistic} for a set of clusters $(C_S^{l, h}, C_A^{l, h}, C_I^{l, h})$ to be 
$$S_{l, h, k} = \frac{\sigma^2_{A, l, h, k} - \sigma^2_{S, l, h, k}}{\sigma^2_{I, l, h, k}},$$
where $\sigma^2_{A, l, h, k}$, $\sigma^2_{S, l, h, k}$, and $\sigma^2_{I, l, h, k}$ are the k-dimensional cluster variances of $C_S^{l, h}$, $C_A^{l, h}$, and $C_I^{l, h}$ respectively for layer $l$ and head $h$.
If $S_{l, h, k}$ is very positive, then the model has learned the stereotype sets of discourse better than the anti-stereotype sets of discourse, and vice versa. The variance of the irrelevant cluster, $\sigma^2_{I, l, h, k}$, serves as a benchmark for the largest variance as the model should have learned no significant topological features for the irrelevant cluster.

\subsection{Permutation test and Wasserstein bias metric}

To perform a permutation test on the Wasserstein Bias Metric, we find the expected value and variance of any given permutation of the values.

\begin{figure*}[t]
  \centering
  \includegraphics[width=.6\textwidth]{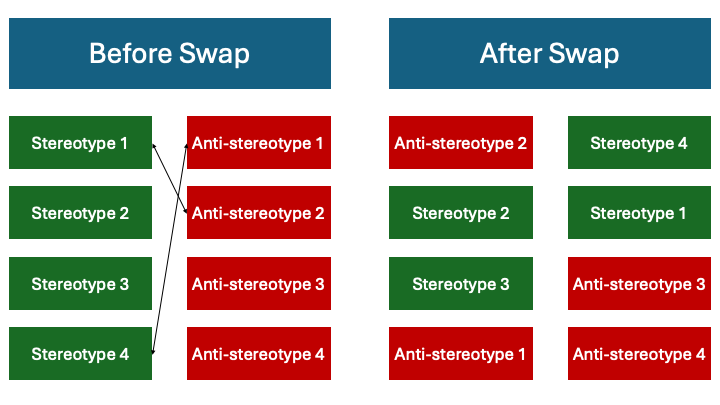}
  \caption{An example of a swap of length 2 in our permutation test.}
  \label{fig:swapFigure}
\end{figure*}

We define a random process where we randomly choose $1 \leq t \leq |C_S^{l, h}|$ stereotypes and anti-stereotype pseudo-inverse distributions to swap as shown in Figure \ref{fig:swapFigure}. We let
\begin{align*}
    W_t = c_1 + ... + c_t,
\end{align*}
where $c_i$ is the change to the statistic caused by the $i$-th swap. 
We can prove the following proposition, which shows us that the value of each $c_i$ is not dependent on the value of $S_{l, h, k}$ or the other swaps. We have shown the proof in Appendix \ref{appendix:statShiftProofs}.

\paragraph{Proposition \ref{appendix:propA5}.} For arbitrary clusters~$(B_S, B_A, B_I)$, if we swap $x \in B_S$ with $y \in B_A$ to get $(B'_S, B'_A, B'_I)$, the value of the statistic for the new clusters $(B'_S, B'_A, B'_I)$ is 
\begin{align*}
    S' & = S + \frac{1}{|B_S|r_I}[d(x, \bar{b}_A) + d(x, \bar{b}_S) \\
    & - d(y, \bar{b}_A) - d(y, \bar{b}_S)],
\end{align*}
where $r_I$ is the variance of the irrelevant cluster $B_I$, and $x$ is the approximated pseudo-inverse distribution from $B_S$ swapped with the approximated pseudo-inverse distribution $y$ from $B_A$.

It can be shown that each swap is equally likely to be in any given position, making them independent and identically distributed. Therefore, by the central limit theorem, we can say that the distribution $(W_t | t)$ is approximately normal.

We can find the distribution of $W_t$ given $t$ by using a swap matrix $A \in R^{n \times n}$ where $n = |C_S^{l, h}|$ is the cluster size and $A_{i, j} = S' - S$ where $S'$ is the new statistic after swapping the $i$-th stereotype with the $j$-th anti-stereotype. The proofs for these propositions are in Appendix \ref{appendix:statShiftDistribution}.

\paragraph{Proposition \ref{appendix:propB1}.} The probability of generating a swap of length $t$ is $P(t) = \frac{N_t}{\sum_{1 \leq k \leq n} N_k}$ where $N_i = \binom{|n|}{i}^2 i!$ for $1 \leq i \leq |n|$.

\paragraph{Proposition \ref{appendix:propB2}.} The expected value of a swap given the length is $E(W_t | t) = t\bar{A}$.

\paragraph{Proposition \ref{appendix:propB3}.} The variance of a swap given the length of the swap is 
\begin{align*}
    \text{Var}(W_t | t) & = E(W_t^2 | t) - E(W_t)^2 \\
                        & = t[\text{Var}(A) + \bar{A}^2] \\ 
                        & + (\frac{t^2 - t}{n^2 (n - 1)^2}) \\
                        & [(\sum_{s,t} A_{s, t})^2 \\ 
                        & - \sum_s (\sum_t A_{s, t})^2 - \sum_t (\sum_s A_{s, t})^2 \\
                        & + \sum_{s, t} A_{s, t}^2] \\ 
                        & - (t\bar{A})^2,
\end{align*}
where $\bar{A} = \frac{1}{n} \sum_{s, t} A_{s, t}$ and $\text{Var}(A) = \frac{1}{n^2} \sum_{s, t} (A_{s, t} - \bar{A})^2$.
Here we will refer to them as $\mu_t = E(W_t | t)$ and $\sigma_t^2 = \text{Var}(W_t | t)$.

We use this conditional distribution to determine the probability of the permuted clusters having a statistic greater than the observed value of our statistic
\begin{align*}
        P(W_t > 0) & = \sum_{t = 1}^{n} P(W_t > 0 | t) P(t) \\
                   & = \sum_{t = 1}^{n} (1 - \Phi(\frac{- \mu_t}{\sigma_t})P(t)),
\end{align*}
where $\Phi$ is the cumulative distribution function of the standard normal distribution.

For our permutation test, we calculate our p-value as
\[ p = \begin{cases} 
      P(W_t > 0) & S_{l, h, k} >= 0 \\
      P(W_t < 0) = 1 - P(W_t > 0) & S_{l, h, k} < 0.
   \end{cases}
\]
Using our permutation test, we can create a k-dimensional Wasserstein Bias Metric for a given set of clusters $(C_S^{l, h}, C_A^{l, h}, C_I^{l, h})$ for a layer $l$ and head $h$ to be $T_{l, h, k} = (1 - p)S_{l, h, k}$ where $p$ is the p-value of the permutation test performed on the cluster for the Wasserstein Statistic $S_{l, h, k}$. This metric allows us to incorporate how confident we are in the value of our statistic.

\subsection{Combined Wasserstein bias metric}

The 0-dimensional Wasserstein Bias Metric captures the differences in which words the model learns connections between, while the 1-dimensional Wasserstein Bias Metric captures the differences in the strength of these connections. Ideally, we would like a metric that combines both of these features so we average both metrics into the \emph{Combined Wasserstein Bias Metric}
\begin{align*}
    S_{l, h} = \frac{S_{l, h, 0} + S_{l, h, 1}}{2}.
\end{align*}

\section{Results and Discussion}
\label{section:results}

\subsection{Metric heat maps}

We can use our metric to show how much each head of the model contributes to the misrepresentation of a given category. In Figures \ref{fig:metricZscoreGender}, \ref{fig:metricZscoreProfession}, \ref{fig:metricZscoreRace}, and \ref{fig:metricZscoreReligion}, we calculate the z-score of the absolute value of the combined metric for each head within a group and then average the z-scores across groups within each category.

\begin{figure*}[t]
  \centering
  \includegraphics[width=.76\textwidth]{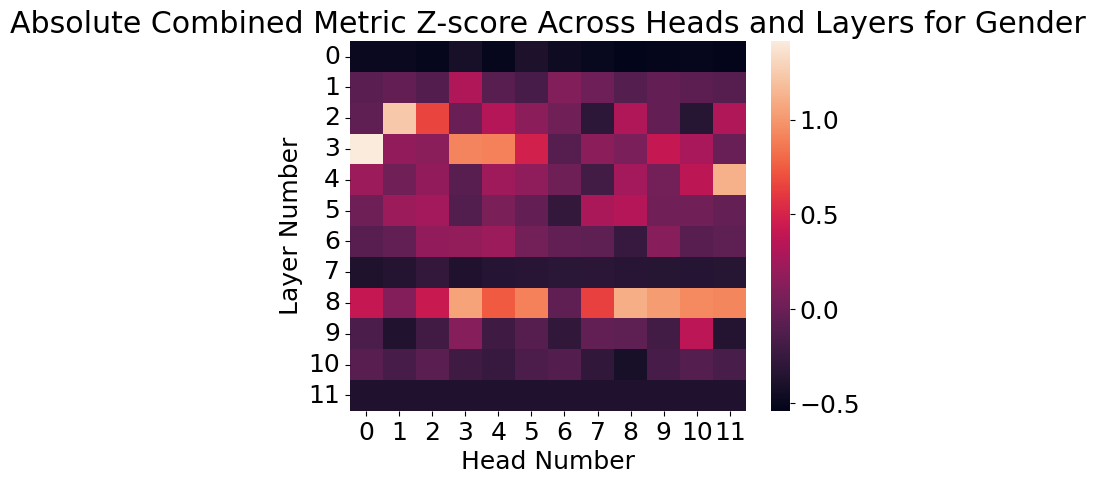}
  \caption{The average z-score of the absolute value of the Combined Wasserstein Bias Metric across groups in the gender category.}
  \label{fig:metricZscoreGender}
\end{figure*}

\begin{figure*}[t]
  \centering
  \includegraphics[width=.79\textwidth]{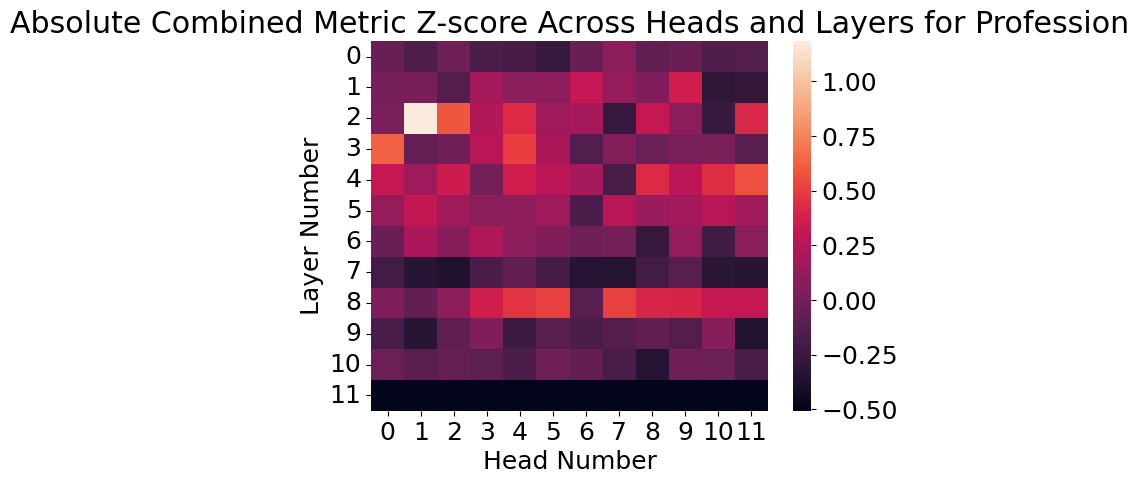}
  \caption{The average z-score of the absolute value of the Combined Wasserstein Bias Metric across groups in the profession category.}
  \label{fig:metricZscoreProfession}
\end{figure*}

\begin{figure*}[t]
  \centering
  \includegraphics[width=.74\textwidth]{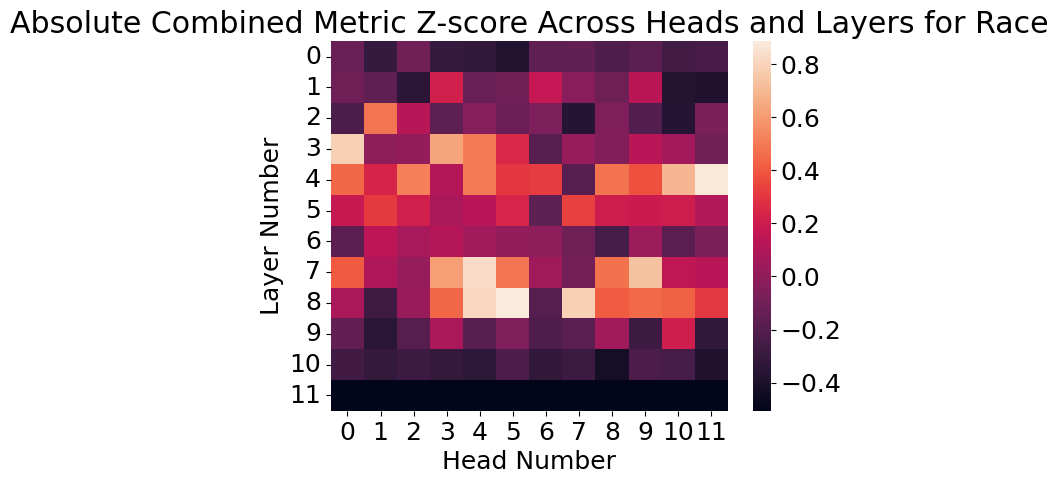}
  \caption{The average z-score of the absolute value of the Combined Wasserstein Bias Metric across groups in the race category.}
  \label{fig:metricZscoreRace}
\end{figure*}

\begin{figure*}[t]
  \centering
  \includegraphics[width=.77\textwidth]{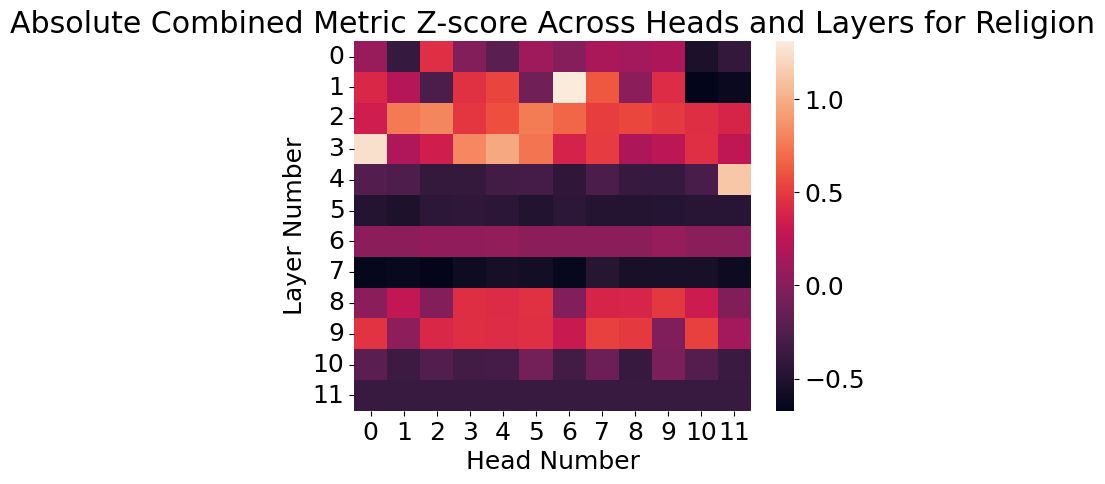}
  \caption{The average z-score of the absolute value of the Combined Wasserstein Bias Metric across groups in the religion category.}
  \label{fig:metricZscoreReligion}
\end{figure*}

For each category, we observe hot spots in the model where there is more bias in that particular category. We also note in Figures \ref{fig:metricZscoreGender} and \ref{fig:metricZscoreProfession} that the hot spots for gender and profession appear closer together, suggesting that their misrepresentation may be connected, similar to what is found by \citep{NEURIPS2021_1531beb7}.

Similar heat maps can be generated for any subgroup within a bias category, allowing us to identify which heads of the model contribute to the misrepresentation of certain groups. In the future, debiasing algorithms could target these heads in order to reduce bias without diminishing language modeling performance.

\subsection{Metric distributions across bias categories}

\begin{figure*}[t]
  \centering
  \includegraphics[width=.62\textwidth]{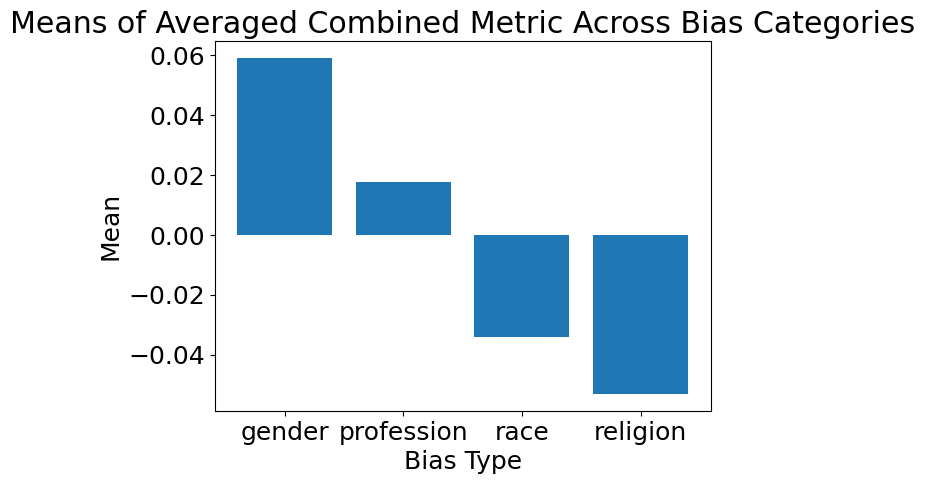}
  \caption{The mean of the averaged Combined Wasserstein Bias Metric across heads for each bias group.}
  \label{fig:meanCombinedMetric}
\end{figure*}

\begin{figure*}[t]
  \centering
  \includegraphics[width=.75\textwidth]{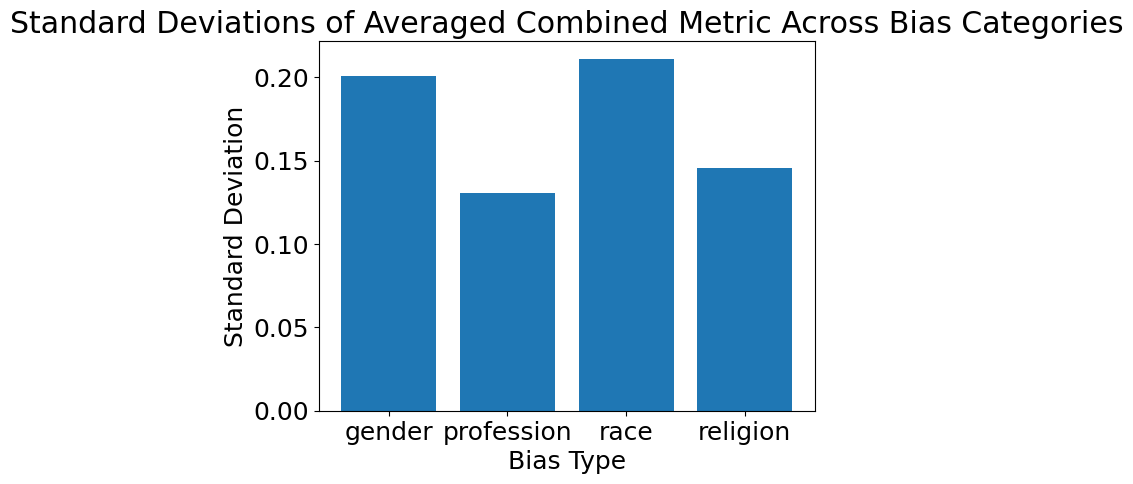}
  \caption{The standard deviation of the averaged Combined Wasserstein Bias Metric across heads for each bias group.}
  \label{fig:stdCombinedMetric}
\end{figure*}

We first observe from Figure \ref{fig:meanCombinedMetric} that the average value of the metric for each group is not significantly different from 0, indicating that, on average, there is no significant misrepresentation of these bias groups as a whole. This is consistent with the bias found by \citet{nadeem-etal-2021-stereoset} using the StereoSet dataset, with GPT-2 preferring stereotypes over anti-stereotypes only 52\% of the time. However, considering only the mean of the metrics ignores potential unfairness between subgroups within a bias category.

From Figure \ref{fig:stdCombinedMetric}, we see that the standard deviation of the observed metric between classes within each bias group is an order of magnitude greater than the mean, with some groups being represented much more negatively than others.

\subsection{Comparison of metrics across bias sub-groups}

We also examine the groups in each category with the most positive and most negative learned representations for both the number of connected components (0D) and the strength of connections (1D), as shown in Tables \ref{tab:gender}, \ref{tab:profession}, \ref{tab:race}, and \ref{tab:religion}. The more positive the value of the combined metric, the more negative the representation of the group.

\begin{table}
  \centering
  \begin{tabular}{lrrr}
    \toprule
    \textbf{Group Name} & \textbf{Mean} & \textbf{Min} & \textbf{Max} \\
    \midrule
    Sister     & 1.942 & -1.173 & 18.970          \\
    Herself     & 1.662 & -2.321 & 23.892           \\
    Schoolboy     & -0.122 & -2.763 & 3.529           \\
    Schoolgirl     & -0.467 & -8.697 & 6.679           \\
    \bottomrule
  \end{tabular}
  \caption{The mean, minimum, and maximum of the Combined Wasserstein Metric across heads for the \textit{gender} categories with the 2 largest and 2 smallest mean Combined Wasserstein Metric.}
  \label{tab:gender}
\end{table}

\begin{table}
  \centering
  \begin{tabular}{llll}
    \toprule
    \textbf{Group Name} & \textbf{Mean} & \textbf{Min} & \textbf{Max} \\
    \midrule
    Psychologist     & 1.615 & 0.030 & 10.572         \\
    Commander     & 0.951 & -10.374 & 14.972           \\
    Software Developer     & -0.913 & -13.235 & 1.760           \\
    Assistant     & -1.387 & -15.730 & 2.338           \\
    \bottomrule
  \end{tabular}
  \caption{The mean, minimum, and maximum of the Combined Wasserstein Metric across heads for the \textit{profession} categories with the 2 largest and 2 smallest mean Combined Wasserstein Metric.}
  \label{tab:profession}
\end{table}

\begin{table}
  \centering
  \begin{tabular}{llll}
    \toprule
    \textbf{Group Name} & \textbf{Mean} & \textbf{Min} & \textbf{Max} \\
    \midrule
    Sierra Leon     & 1.937 & -6.227 & 23.996          \\
    Cape Verde     & 1.047 & -1.124 & 7.943           \\
    Lebanon     & -1.492 & -13.279 & 0.348           \\
    Norway     & -1.638 & -19.774 & 6.098           \\
    \bottomrule
  \end{tabular}
  \caption{The mean, minimum, and maximum of the Combined Wasserstein Metric across heads for the \textit{race} categories with the 2 largest and 2 smallest mean Combined Wasserstein Metric.}
  \label{tab:race}
\end{table}

\begin{table}
  \centering
  \begin{tabular}{llll}
    \toprule
    \textbf{Group Name} & \textbf{Mean} & \textbf{Min} & \textbf{Max} \\
    \midrule
    Brahmin     & 0.419 & -0.164 & 2.232          \\
    Muslim     & 0.266 & -1.494 & 2.324           \\
    Bible     & -0.416 & -5.563 & 0.964           \\
    \bottomrule
  \end{tabular}
  \caption{The mean, minimum, and maximum of the Combined Wasserstein Metric across heads for all \textit{religion} categories.}
  \label{tab:religion}
\end{table}

A notable observation is that there is a large range of metric values for some groups, indicating that some of the heads are disproportionately contributing to the misrepresentation of these groups. We can see these heads on the metric heat maps. These hot spots appear to be larger in the race and gender categories.

\section{Conclusions and Future Work}
\label{section:conclusions}

Our study introduces the Wasserstein bias metric to determine the level of misrepresentation that self-attention-based large language models learn in particular heads. We have demonstrated that GPT-2 significantly misrepresents gender and race identity categories and have also developed a method to identify which heads contribute to the misrepresentation of a particular identity group. Our data show that there are specific heads that act as hot spots for misrepresentation in certain categories.

In the future, this metric could be used to determine which layers and heads of large language models should be fine-tuned to reduce bias without decreasing overall language modeling performance, and future work could investigate the reasons for increased misrepresentation in the beginning and end heads of a layer compared to the middle heads. This method could also be extended to other datasets containing different identity groups, and further improvements may include using longer input sentences or even paragraphs to better capture biases present at the paragraph level in the model.

\section{Limitations}
\label{section:limitations}

One limitation of this study is the use of the StereoSet dataset, which may not include all relevant groups and has known limitations as discussed by \citet{blodgett-etal-2021-stereotyping}. Due to these dataset-specific limitations, the findings presented here are not necessarily conclusive, and we recommend that practitioners adapt the proposed method to datasets developed for their particular domain, rather than using it as a definitive measure of bias. Another limitation is the use of an approximation for the pseudo-inverse distributions of topological features. This approximation may have introduced error into our statistic, but it was necessary to ensure computational tractability. Another challenge lies in interpreting the magnitude of our proposed metric, as it is not directly connected to the actual outputs of the model, such as the probability of given words. Instead, the metric should be used to compare bias between groups, rather than to quantify the absolute magnitude of bias in a single group. Finally, as the dataset contained only English sentences, the results of this study may differ with datasets in other languages, and future work could extend the metric to these cases.

\section{Ethical Considerations}
\label{section:ethics}

Although the methods proposed in this paper may not directly cause harm, the debiasing techniques they motivate could also be used to increase bias in large language models for particular groups. Care should be taken to prevent such misuse when applying these techniques.

\bibliographystyle{plainnat}
\bibliography{reference}

\clearpage

\appendix

\section{Proofs}\label{appendix:statShiftProofs}

For the following proofs, let $B$ be an arbitrary set of approximated pseudo-inverse distributions with smoothing parameter $n$. 

Similarly, let $B_A$, $B_S$, and $B_I$ be sets of approximated pseudo-inverse distributions for clusters of anti-stereotypes, stereotypes, and irrelevant attention graphs respectively.

Let $f_k = f(x)$ where $\frac{k}{n} \leq x < \frac{k + 1}{n}$ be the value of the approximated pseudo-inverse distribution $f$ at step $0 \leq k \leq n - 1$.

Let $x$ be an approximated pseudo-inverse distribution added to $B$ and $y$ be an approximated pseudo-inverse distribution removed from $B$ to obtain $B'$.

Let $\bar{b}$ be the cluster center of $B$ and let $\bar{b'}$ be the cluster center of $B'$. Let $r$ and $r'$ be the cluster variances of $B$ and $B'$ respectively.

We define $d(f, g)$ to be the average square difference between two functions
\begin{align*}
    d(f, g) & = \int_0^1 (f(z) - g(z))^2 dz \\ 
    & = \frac{1}{n} \series{0}{n - 1} (f_k - g_k)^2.
\end{align*}

Let $S$ and $S'$ be the values of the Wasserstein bias statistic generated for the set of clusters $(B_S, B_A, B_I)$ and $(B'_S, B'_A, B'_I)$ respectively.

\begin{proposition}\label{appendix:propA1}
    For any step $k$, $\bar{b}_k - \bar{b'}_k = \frac{y_k - x_k}{|B|}$.
\end{proposition}

\begin{proof}
    Observe that
    \[
        \bar{b'}_k = \frac{|B|\bar{b}_k + x_k - y_k}{|B|} = \bar{b}_k + \frac{x_k - y_k}{|B|}.
    \]
    This leads us to find that 
    \begin{align*}
        \bar{b}_k - \bar{b'}_k & = \bar{b}_k - \bar{b}_k - \frac{x_k - y_k}{|B|} \\
        & = - \frac{x_k - y_k}{|B|} = \frac{y_k - x_k}{|B|}.
    \end{align*}
\end{proof}

\begin{proposition}\label{appendix:propA2}
    For every $b \in B'$, 
    \begin{align*}
        d(b, \bar{b'}) = d(b, \bar{b}) &+ \frac{1}{|B|^2} d(x, y) \\
        &+ \frac{2}{n} \series{0}{n - 1} \frac{(y_k - x_k)(b_k - \bar{b}_k)}{|B|}.
    \end{align*}
\end{proposition}

\begin{proof}
    Let $k$ be an arbitrary step. Observe that the squared difference between the old step and the new step mean is
    \begin{align*}
        (b_k - \bar{b'}_k)^2 & = ((b_k - \bar{b}_k) + (\bar{b}_k - \bar{b'}_k))^2 \\
        & = (b_k - \bar{b}_k)^2 + (\bar{b}_k - \bar{b'}_k)^2 \\ & + 2(b_k - \bar{b}_k)(\bar{b}_k - \bar{b'}_k).
    \end{align*}
    So the distance between an old pseudo-inverse distribution and the new mean is
    \begin{align*}
        d(b, \bar{b'}) & = \frac{1}{n}\series{0}{n - 1} (b_k - \bar{b'}_k)^2 \\ 
        & = \frac{1}{n}\series{0}{n - 1} [(b_k - \bar{b}_k)^2 + (\bar{b}_k - \bar{b'}_k)^2 \\ 
        & + 2(b_k - \bar{b}_k)(\bar{b}_k - \bar{b'}_k)] \\
        & = \frac{1}{n} \series{0}{n - 1}(b_k - \bar{b}_k)^2 + \frac{1}{n} \series{1}{n}(\bar{b}_k - \bar{b'}_k)^2 \\ 
        & + \frac{1}{n} \series{0}{n - 1} 2(b_k - \bar{b}_k)(\bar{b}_k - \bar{b'}_k).
    \end{align*}
    Using Proposition \ref{appendix:propA1}, we can substitute so 
    \begin{align*}
        d(b, \bar{b'}) & = d(b, \bar{b}) + \frac{1}{n} \series{0}{n - 1} \frac{(y_k - x_k)^2}{|B|^2} \\ & + \frac{2}{n} \series{0}{n - 1} \frac{(y_k - x_k)(b_k - \bar{b}_k)}{|B|}.
    \end{align*}
    From the definition of the distance function, $\frac{1}{n} \series{0}{n - 1} (y_k - x_k)^2 = d(x, y)$ so
    \begin{align*}
        d(b, \bar{b'}) & = d(b, \bar{b}) + \frac{1}{|B|^2} d(x, y) \\ 
        & + \frac{2}{n} \series{0}{n - 1} \frac{(y_k - x_k)(b_k - \bar{b}_k)}{|B|}.
    \end{align*}
    
\end{proof}

\begin{proposition}\label{appendix:propA3}
    The sum $$\sum_{b \in B} [\frac{2}{n} \series{0}{n - 1} \frac{(y_k - x_k)(b_k - \bar{b}_k)}{|B|}] = 0.$$
\end{proposition}

\begin{proof}
    Observe that we can move the outside summation inwards so
    \begin{align*}
        \sum_{b \in B} [\frac{2}{n}& \series{0}{n - 1} \frac{(y_k - x_k)(b_k - \bar{b}_k)}{|B|}] \\
        &= \frac{2}{n}[\series{0}{n - 1} \frac{(y_k - x_k)(\sum_{b \in B}(b_k - \bar{b}_k))}{|B|}]
    \end{align*}
    because $\sum_{b \in B}(b_k - \bar{b}_k)$ is the sum of the differences from the mean, $\sum_{b \in B}(b_k - \bar{b}_k) = 0$.
    Thus we can simplify our expression to
    \begin{align*}
        \frac{2}{n}[\series{0}{n - 1}& \frac{(y_k - x_k)(\sum_{b \in B}(b_k - \bar{b}_k))}{|B|}] \\ 
        &= \frac{2}{n}[\series{0}{n - 1} \frac{(y_k - x_k)(0)}{|B|}] = 0.
    \end{align*}
    Thus we have shown that $$\sum_{b \in B} [\frac{2}{n} \series{0}{n - 1} \frac{(y_k - x_k)(b_k - \bar{b}_k)}{|B|}] = 0.$$
\end{proof}

\begin{proposition}\label{appendix:propA4}
    The variance of the new cluster is $$r' = r + \frac{1}{|B|}[d(x,\bar{b}) - d(y,\bar{b}) - \frac{1}{|B|}d(x,y)].$$
\end{proposition}

\begin{proof}
    Observe that by definition $$|B|r' = \sum_{b \in B}[d(b, \bar{b'})] + d(x, \bar{b'}) - d(y, \bar{b'}).$$
    Using Proposition \ref{appendix:propA2}, we can substitute and have
    \begin{align*}
        |B|r' & = \sum_{b \in B}[d(b, \bar{b}) + \frac{1}{|B|^2} d(x, y) \\ 
        & + \frac{2}{n} \series{0}{n - 1} \frac{(y_k - x_k)(b - \bar{b}_k)}{|B|}] \\ 
        & + [d(x, \bar{b}) + \frac{1}{|B|^2} d(x, y) \\
        & + \frac{2}{n} \series{0}{n - 1} \frac{(y_k - x_k)(x_k - \bar{b}_k)}{|B|}] \\ 
        & - [d(y, \bar{b}) + \frac{1}{|B|^2} d(x, y) \\
        & + \frac{2}{n} \series{0}{n - 1} \frac{(y_k - x_k)(y_k - \bar{b}_k)}{|B|}].
    \end{align*}
    Applying Proposition \ref{appendix:propA3} to simplify and combining like terms, we have
    \begin{align*}
        |B|r' & = \sum_{b \in B}[d(b, \bar{b}) + \frac{1}{|B|^2} d(x, y)] \\ 
        & + d(x, \bar{b}) - d(y, \bar{b}) - \frac{2}{|B|} \series{0}{n - 1}\frac{(y_k - x_k)^2}{n}.
    \end{align*}
    Further simplifying gives us,
    \begin{align*}
        |B|r' & = \sum_{b \in B} d(b, \bar{b}) + \sum_{b \in B} \frac{1}{|B|^2} d(x,y) \\ 
        & + d(x,\bar{b}) - d(y,\bar{b}) - \frac{2}{|B|} d(x, y) \\
        & = \sum_{b \in B} d(b, \bar{b}) + \frac{1}{|B|} d(x,y) \\
        & + d(x,\bar{b}) - d(y,\bar{b}) - \frac{2}{|B|} d(x, y) \\
        & = \sum_{b \in B} d(b, \bar{b}) + d(x,\bar{b}) - d(y,\bar{b}) - \frac{1}{|B|} d(x, y).
    \end{align*}
    Dividing both sides by $|B|$ gives us, $$r' = r + \frac{1}{|B|}[d(x, \bar{b}) - d(y, \bar{b}) - \frac{1}{|B|} d(x, y)].$$

\end{proof}

\begin{proposition}\label{appendix:propA5}
The value of the statistic for the new clusters $(B'_S, B'_A, B'_I)$ is 
\begin{align*}
    S' & = S + \frac{1}{|B_S|r_I}[d(x, \bar{b}_A) + d(x, \bar{b}_S) \\
    & - d(y, \bar{b}_A) - d(y, \bar{b}_S)],
\end{align*}
where $r_I$ is the variance of the irrelevant cluster $B_I$, and $x$ is the approximated pseudo-inverse distribution from $B_S$ swapped with the approximated pseudo-inverse distribution $y$ from $B_A$.
\end{proposition}

\begin{proof}
    Using Proposition \ref{appendix:propA4}, the anti-stereotype cluster variance can be calculated as 
    \begin{align*}
        r'_A & = r_A + \frac{1}{|B_S|}[d(x, \bar{b}_A) - d(y, \bar{b}_A) \\ 
        & - \frac{1}{|B_S|}d(x,y)].
    \end{align*}
    Similarly, the stereotype cluster variance is 
    \begin{align*}
        r'_S & = r_S + \frac{1}{|B_S|}[d(y, \bar{b}_S) - d(x, \bar{b}_S) \\ 
        & - \frac{1}{|B_S|}d(x,y)].
    \end{align*}
    Then, by the definition of the Wasserstein Bias Statistic 
    \begin{align*}
        S' & = \frac{r'_A - r'_S}{r_I} \\ 
        & = \frac{r_A - r_S}{r_I} \\ 
        & + \frac{1}{|B_S|r_I} [d(x, \bar{b}_A) + d(x, \bar{b}_S) \\ 
        & - d(y, \bar{b}_A) - d(y, \bar{b}_S)] \\ 
        & = S + \frac{1}{|B_S|r_I} [d(x, \bar{b}_A) + d(x, \bar{b}_S) \\ 
        & - d(y, \bar{b}_A) - d(y, \bar{b}_S)].
    \end{align*}
\end{proof}

\section{Proofs}\label{appendix:statShiftDistribution}

Let $S$ be the statistic calculated for the clusters $(C_S, C_A, C_I)$ and let $n = |C_S|$ be the cluster size.

Let $A \in R^{n \times n}$ be the swap matrix, where $A_{i, j}$ is the change in the statistic caused by swapping the $i$-th stereotype in $C_S$ with the $j$-th anti-stereotype in $C_A$.

We define $\bar{A} = \frac{1}{n^2} \sum_{s, t} A_{s, t}$ as the average of all elements of $A$. Similarly, $\text{Var}(A)$ is the variance of the elements of the matrix $A$. 

We define $A_{i, :}$ as the $i$-th row and $A_{:, j}$ as the $j$-th column. We define $A^{[i, j]}$ as the matrix obtained by removing the $i$-th row and $j$-th column of $A$.

Define a random process where we randomly choose $1 \leq t \leq |C_S|$ stereotypes and anti-stereotype pseudo-inverse distributions to swap. We let $W_t = c_1 + ... + c_t$ where $c_i$ is the change to the statistic caused by the $i$-th swap. By finding the distribution of $W_t$, we can determine the how the Wasserstein Bias Statistic will change as we permute our clusters.

\begin{proposition}\label{appendix:propB1}
    The probability of generating a swap of length $t$ is $P(t) = \frac{N_t}{\sum_{1 \leq k \leq n} N_k}$ where $N_i = \binom{n}{i}^2 i!$ for $1 \leq i \leq n$.
\end{proposition}

\begin{proof}
    We first determine that the number of swaps of length $t$ is $N_t = \binom{n}{t}^2 t!$. This is because there are $\binom{n}{t}$ different ways to choose the stereotype to swap and $\binom{n}{t} * t!$ different ways to match anti-stereotypes to each stereotype. Therefore the distribution of $t$ is $$P(t) = \frac{N_t}{\sum_{1 \leq k \leq n} N_k}.$$
\end{proof}

\begin{proposition}\label{appendix:propB2}
    The expected value of a swap given the length is $E(W_t | t) = t\bar{A}$.
\end{proposition}

\begin{proof}
    This result stems directly from the fact that each entry of the swap matrix $A$ has an equal probability of being in the $i$-th swap, making the expected value of a single swap $\bar{A}$. There are $t$ such swaps, meaning that $E(W_t | t) = t\bar{A}$. 
\end{proof}

\begin{proposition}\label{appendix:propB3}
    The variance of a swap given the length of the swap is 
    \begin{align*}
        \text{Var}(W_t | t) & = E(W_t^2 | t) - E(W_t)^2 \\
                            & = t[\text{Var}(A) + \bar{A}^2] \\ 
                            & + (\frac{t^2 - t}{n^2 (n - 1)^2}) \\
                            & [(\sum_{s,t} A_{s, t})^2 \\ 
                            & - \sum_s (\sum_t A_{s, t})^2 - \sum_t (\sum_s A_{s, t})^2 \\
                            & + \sum_{s, t} A_{s, t}^2] \\ 
                            & - (t\bar{A})^2.
    \end{align*}
\end{proposition}

\begin{proof}
    First, we find the expected value of $W_t^2$ as
    \begin{align*}
        E(W_t^2 | t) & = E((c_1 + c_2 + ... + c_t)^2 | t) \\
        & = E(\sum_{i = 1}^{t} c_i^2 + \sum_{i \neq j} c_i c_j | t) \\
        & = E(\sum_{i = 1}^{t} c_i^2 | t) + E(\sum_{i \neq j} c_i c_j | t).
    \end{align*}
    We find the expected value of $c_i^2$ using the formula for variance
    $$\text{Var}(c_i | t) = E(c_i^2 | t) - E(c_i | t)^2.$$
    This gives us
    \begin{align*}
        E(c_i^2 | t) & = \text{Var}(c_i | t) +  E(c_i | t)^2 \\
        & = \text{Var}(c_i) +  E(c_i)^2 \\
        & = \text{Var}(A) + \bar{A}^2.
    \end{align*}
    Also observe because each pair $c_i, c_j$ is chosen independently of the number of swaps, 
    \begin{align*}
        E(c_i c_j | t) & = E(c_i c_j) = \frac{1}{n^2} \sum_{s, t} E(c_i c_j | c_i = A_{s, t}) \\ 
        & = \frac{1}{n^2} \sum_{s, t} A_{s, t} E(c_j | c_j \notin A_{s, :} \text{ and } c_j \notin A_{:, t}) \\
        & = \frac{1}{n^2} \sum_{s, t} A_{s, t} \bar{A}^{[s, t]}.
    \end{align*}
    We can simplify this into an expression we can compute in terms of A as
    \begin{align*}
        &E(c_i c_j | t) = \frac{1}{n^2} \sum_{s, t} A_{s, t} \bar{A}^{[s, t]} \\
                       & = \frac{1}{n^2} \sum_{s, t} \frac{A_{s, t}}{(n - 1)^2} (\sum_{s, t} A_{s,t} - \sum_t A_{s, t} \\
                       & - \sum_s A_{s, t} + A_{s, t}) \\
                       & = \frac{1}{n^2 (n - 1)^2} [\sum_{s, t} A_{s, t} \sum_{s, t} A_{s, t} \\
                       & - \sum_{s,t} A_{s, t} \sum_t A_{s, t} \\ 
                       & - \sum_{s,t} A_{s, t} \sum_s A_{s, t} + \sum_{s, t} A_{s, t}^2] \\
                       & = \frac{1}{n^2 (n - 1)^2} [(\sum_{s, t} A_{s, t}) \sum_{s, t} A_{s, t} \\
                       & - \sum_s (\sum_t A_{s, t}) \sum_t A_{s, t} - \sum_t (\sum_s A_{s, t}) \sum_s A_{s, t} \\
                       & + \sum_{s, t} A_{s, t}^2] \\
                       & = \frac{1}{n^2 (n - 1)^2} [(\sum_{s, t} A_{s, t})^2 - \sum_s (\sum_t A_{s, t})^2 \\
                       & - \sum_t (\sum_s A_{s, t})^2 + \sum_{s, t} A_{s, t}^2].
    \end{align*}
    Combining the values of $E(c_i^2 | t)$ and $E(c_i c_j | t)$, we have 
    \begin{align*}
        E(W_t^2 | t) & = E(\sum_{i = 1}^{t} c_i^2 | t) + E(\sum_{i \neq j} c_i c_j | t) \\
                     & = \sum_{i = 1}^t E(c_i^2 | t) + \sum_{i \neq j} E(c_i c_j | t) \\
                     & = t[\text{Var}(A) + \bar{A}^2] \\
                     & + (\frac{t^2 - t}{n^2 (n - 1)^2}) \\
                     & [(\sum_{s,t} A_{s,t})^2 \\ 
                     & - \sum_s (\sum_t A_{s, t})^2 - \sum_t (\sum_s A_{s, t})^2 \\
                     & + \sum_{s, t} A_{s, t}^2].
    \end{align*}
    Combining this with the equation for variance we have
    \begin{align*}
        \text{Var}(W_t | t) & = E(W_t^2 | t) - E(W_t)^2 \\
                            & = t[\text{Var}(A) + \bar{A}^2] \\ 
                            & + (\frac{t^2 - t}{n^2 (n - 1)^2}) [(\sum_{s,t} A_{s,t})^2 - \sum_s (\sum_t A_{s, t})^2 \\ & 
                            - \sum_t (\sum_s A_{s, t})^2 + \sum_{s, t} A_{s, t}^2] \\ 
                            & - (t\bar{A})^2.
    \end{align*}
\end{proof}

\end{document}